\documentclass[letterpaper, 10 pt, conference]{ieeeconf}  

\usepackage[utf8]{inputenc}         
\usepackage[T1]{fontenc}            
\usepackage[hidelinks]{hyperref}    
\usepackage{url}                    
\usepackage{booktabs}               
\usepackage{amsfonts}               
\usepackage{nicefrac}               
\usepackage{xcolor}                 
\usepackage{graphicx}
\usepackage{amsmath,amssymb}
\usepackage{hyperref}
\usepackage{paralist}
\usepackage{subcaption}

\newtheorem{lem}{Lemma}

\newtheorem{rmk}{Remark}

\def\R{\mathbb{R}}

\def\mc{\mathcal}

\IEEEoverridecommandlockouts                 
\title{\LARGE \bf
Model-Free Adaptive Parameter Tuning for \\Efficient Multi-Robot Warehouse Operations
}

\author{Pratap Tokekar, Mouhacine Benosman, Rahul Chandan,\\ Alexandre Ormiga Galvao Barbosa, Michael Caldara, Joseph W. Durham%
    \thanks{The authors are with Amazon Robotics. Emails: {\tt \{tokekar, mbenos,}
    {\tt rcd, aormiga, caldaram, josepdur\}@amazon.com}}%
    \thanks{Tokekar holds concurrent appointments as an Associate Professor of Computer Science at the University of Maryland and as an Amazon Scholar. This paper describes work performed at Amazon and is not associated with the University of Maryland.}%
}

\begin{document}
\maketitle
\thispagestyle{empty}
\pagestyle{empty}

\begin{abstract}
Robotic Fulfillment Centers (FCs) store inventory on shelves (pods) arranged in dense blocks. Retrieving a target pod that is buried deep in a block requires moving obstructing pods out of the way (i.e., digout). Multi-robot planners use parameterized cost functions to control digout behavior, producing a spectrum of strategies: at one extreme, obstructing pods are sent to other blocks (using more robots in travel lanes); at the other, pods are shuffled within the block (avoiding lane congestion but increasing extraction time). Each point on this spectrum has different downstream consequences for floor congestion and throughput. The optimal operating point depends on the specific facility configuration and shifts with operational conditions such as varying station demand and congestion patterns, making offline tuning impractical. We present an adaptive parameter tuning framework based on Extremum Seeking Control (ESC) that continuously adjusts planner parameters in response to measured throughput. ESC performs model-free optimization by perturbing parameters with sinusoidal dither signals and correlating perturbations with performance changes to estimate gradients, making it robust to the multi-minute delayed effects and credit assignment challenges inherent in large FC operations. Simulation studies demonstrate that the adaptive policy improves upon fixed policies across several conditions. We observe an improvement in throughput by an average of 5.0\% across map and robot fleet size variations, and by 8.4\% under dynamic operating conditions. This work eliminates manual parameter provisioning and enables real-time adaptation, providing a self-tuning paradigm for FC storage operations.
\end{abstract}

\section{Introduction}\label{sec:Intro}
Robotic Fulfillment Centers (FCs) rely on fleets of mobile robots to execute large-scale storage and retrieval operations. Robots transport pods (shelves containing inventory) to pick stations where associates retrieve items to fulfill customer orders~\cite{wurman2008coordinating,honig2019persistent}. Pods are arranged in dense storage blocks where target pods may be obstructed by other pods and must be extracted, i.e., ``dug out'', to reach the boundary (Figure~\ref{fig:overview}). A central challenge in these systems is deciding how to perform extraction of pods stored deep within a block. Classical planning algorithms compute digout plans that move obstructing pods
  out of the way~\cite{honig2019persistent,fu_2026,li2021lifelong} by minimizing a parameterized cost function (Figure~\ref{fig:overview}). The cost parameters
  control the planner's behavior along a spectrum of strategies (Figure~\ref{fig:twoways}): at one end, obstructing pods are sent to other blocks, using multiple
  robots in travel lanes; at the other, pods are shuffled within the block using free cells, avoiding lane congestion but requiring longer extraction sequences;
  intermediate values produce hybrid plans. Crucially, each point on this spectrum has different downstream consequences for the floor as a whole: sending pods
  away increases lane traffic and storage contention at block peripheries, while in-block shuffling increases local extraction time.  Together, these parameters determine throughput, the metric of ultimate interest.

  Finding the parameter configuration that maximizes long-term throughput, rather than optimizing any individual plan's makespan or path length, is difficult
  for several reasons. No closed-form model relates parameters to throughput, so optimization requires running time-consuming grid searches over expensive
  simulations. Even when good parameters are found in simulation, they may not transfer directly to real systems due to modeling discrepancies in robot dynamics,
  traffic patterns, and station behavior. Moreover, the optimal operating point depends on the specific facility (map layout, robot count, storage density) and
  shifts with dynamic operational conditions such as varying station demand and congestion patterns. Static parameters tuned offline for one configuration are unlikely to remain optimal when
  conditions change, motivating an adaptive approach.


\begin{figure}[thb]
  \centering
  \includegraphics[height=1.1in]{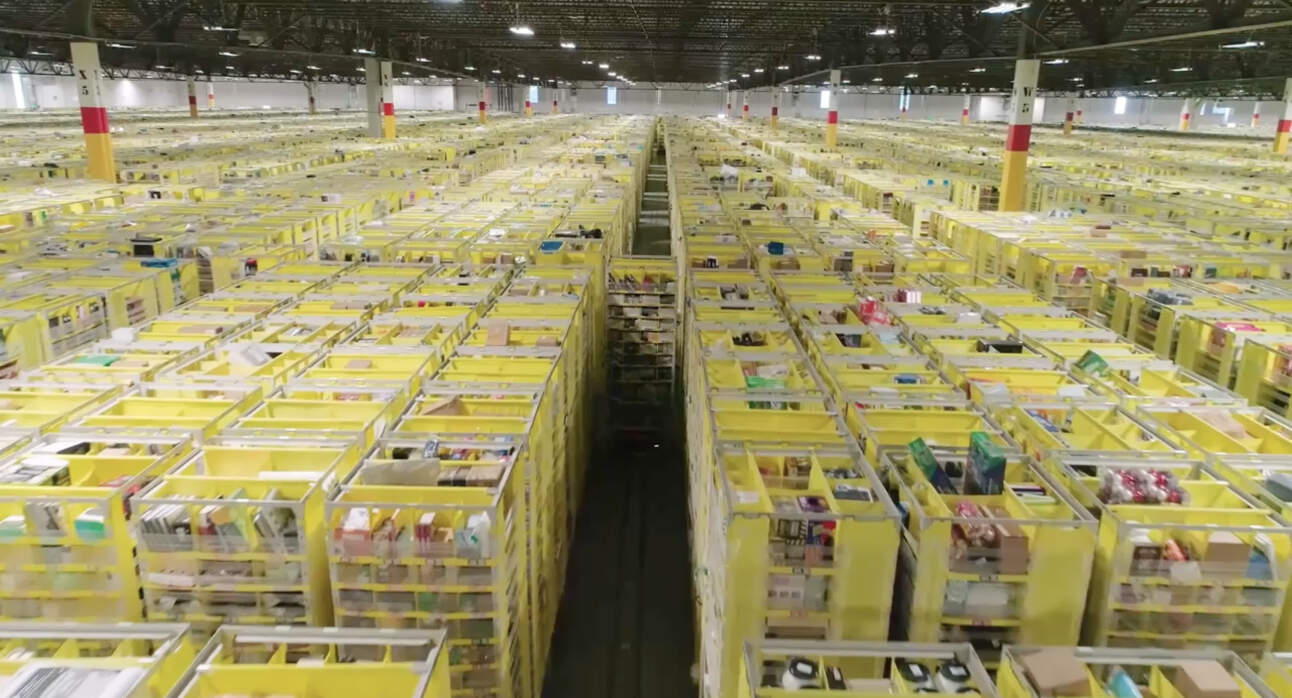}
  \includegraphics[height=1.1in]{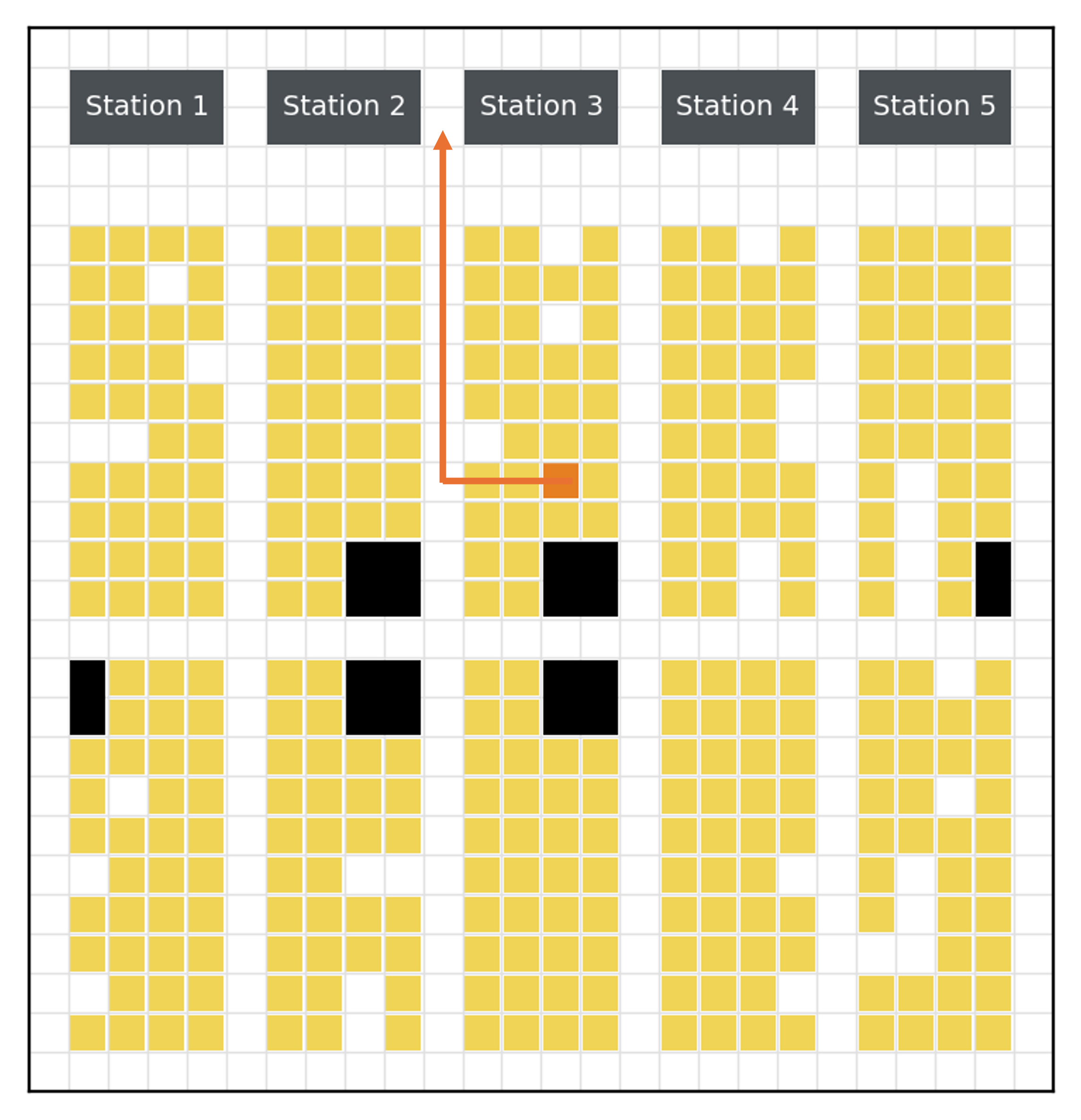}
  \caption{Fulfillment center where the storage pods (yellow) are arranged in grid blocks separated by travel lanes. Robots can lift and move pods to bring them to the stations for item picking or stowing. Target pods (orange) that are buried deep in a block require moving obstructing pods.}
  \label{fig:overview}
\end{figure}

To address these challenges, we propose an adaptive parameter tuning framework based on Extremum Seeking Control (ESC). The framework addresses both automatic parameter provisioning (eliminating grid search) and dynamic parameter adaptation based on real-time feedback. ESC continuously adjusts planner parameters in response to measured throughput, performing derivative-free optimization that is robust to noise, delayed effects, and unmodeled dynamics~\cite{ariyur2003real,krstic2000adaptive,tan2006nonlinear,Benosman2016}. By perturbing parameters with a dither signal and correlating the perturbation with throughput changes, ESC estimates gradients without requiring a closed-form model of the system.

In this paper, we investigate four targeted research questions. First, there is no closed-form model relating parameters to throughput, precluding gradient-based optimization and requiring model-free techniques. A fundamental question is whether parameter perturbations even produce observable throughput changes that can be exploited for optimization (RQ1: \emph{Does dither in cost manifest in throughput?}). Second, parameters only control the local behavior of how a target pod is extracted, but this change ripples through the floor over time. For example, a parameter choice that leads to more robots sending away obstructing pods, rather than shuffling them locally to bubble target pods, might cause congestion in the future, reducing throughput. This delayed effect makes credit assignment of parameter changes difficult (RQ2: \emph{What is the phase delay between parameter changes and throughput response?}). Third, the optimal configuration depends on the specific FC characteristics: number of robots, amount of work, size of the floor, and storage density. The best parameter configuration on one floor may not transfer to another, necessitating a new grid search for each deployment. We investigate whether an adaptive policy can automatically converge to good configurations, eliminating manual parameter sweeps (RQ3: \emph{What is the on-policy performance under nominal conditions?}). Fourth, measuring the effectiveness of a parameter choice requires evaluation on downstream metrics, primarily throughput, which can only be assessed after long simulation runs. Moreover, static parameters cannot adapt when operational conditions change, such as varying station availability or workload patterns (RQ4: \emph{What is the on-policy performance under dynamic conditions?}).

Extensive simulation studies, on a high fidelity discrete event FC simulator, using realistic FC parameters demonstrate that ESC automatically converges to near-optimal parameters across various operational conditions. When fixed policies tuned on one configuration are transferred to different maps or robot counts, they incur an average throughput loss of 5.0\% compared to an ESC policy that adapts online. Under dynamic conditions where 50\% of stations log off mid-operation, ESC improves throughput by 8.4\% overall and 14.3\% in the post-disturbance period by automatically adjusting parameters to match the changed operating regime. This work bridges search-based planning with adaptive control, providing a self-tuning paradigm for dense FC operations.


\section{Related Work}\label{sec_relwork}

\textbf{Extremum Seeking Control.} ESC has a rich history in adaptive control, with foundational work by~\cite{krstic2000adaptive} establishing stability and convergence guarantees under averaging assumptions. \cite{ariyur2003real} extended these results to practical real-time optimization settings. Recent surveys~\cite{tan2006nonlinear, S2024, B2026} highlight applications in automotive, aerospace, and process control, though applications to multi-robot warehouse systems remain limited. Our work extends the application of ESC to the domain of combinatorial planning parameters, where the relationship between cost function weights and system-level throughput is mediated by a large scale discrete event system.

\textbf{Black-Box Optimization for Robotics.} Several derivative-free approaches have been explored for parameter tuning in robotic systems. Bayesian optimization (BO) \cite{SSWAdF2016} constructs probabilistic surrogate models and offers strong sample efficiency, though it typically assumes stationary objectives, and relies on underlying nonlinear optimization algorithms that are not always amenable to simple industrial implementation under compute constraints. Contextual bandits~\cite{AG2013} provide a framework for adaptive decision-making in continuous and discrete settings, though they typically assume immediate reward signals, an assumption that may not hold in FC environments where parameter changes have delayed, cascading effects on throughput. Iterative Feedback Tuning (IFT)~\cite{hjalmarsson1998iterative} offers a data-driven approach to controller parameter optimization using closed-loop experiments to estimate performance gradients. However, IFT requires multiple probes of the system to estimate the cost gradient, which is impractical in real-time and makes it less suitable for continuously varying FC conditions. In contrast, ESC operates in continuous parameter spaces while remaining model-free and offering principled gradient estimation with convergence guarantees, making it well-suited for the non-stationary conditions found in FC operations.

\textbf{Multi-Robot Coordination.} Recent work addresses block rearrangement problem where pods within a
storage block must be moved to their assigned target sets \cite{fu_2026,khan2025multi}. However, adaptive tuning of planner parameters in response to real-time throughput feedback remains underexplored. Our contribution bridges this gap by treating parameter selection as an online feedback gain tuning problem amenable to adaptive control techniques.

\section{Problem Formulation}\label{sec_pf}

An FC floor typically consists of pods containing inventory arranged in grid-like blocks (Figure~\ref{fig:overview}) with several work stations throughout the floor~\cite{wurman2008coordinating,fu_2026}. 
During FC operation, each station generates requests for target pods to be retrieved. The target pods may be obstructed by other pods that must be moved to access and retrieve the target. We can use a discrete search algorithm (specifically, A*~\cite{HNR1968}) to compute digout plans that extract target pods and bring them to requesting stations. The A* planner plans one pod extraction at a time, though each plan may involve several robots allocated to execute the necessary actions. Stations may generate multiple concurrent requests up to a per-station limit. The objective is to maximize throughput: the total number of pod retrievals completed per unit time across all stations.

The planner operates on a graph representation of the storage block, where vertices correspond to storage locations and edges represent adjacency between storage locations. Pods occupy vertices with the constraint that no two pods can occupy the same vertex simultaneously. Since pods cannot move themselves but require robots to lift and lower them, the planner reasons about four atomic actions: SLIDE$(p, u, v)$: Move lifted pod $p$ from vertex $u$ to adjacent vertex $v$, LIFT$(p, u)$: Change state of pod $p$ at vertex $u$ to lifted, LOWER$(p, u)$: Change state of pod $p$ at vertex $u$ to lowered, DIG$(p, u)$: Remove lifted pod $p$ from boundary vertex $u$ (send to another block).

The planner evaluates candidate digout plans using a parameterized cost function. The total cost of a plan, which is fed as the heuristic of the A* planner, is the weighted sum of action counts:
\begin{align}
\label{eq:plan_cost}
F_\theta(\text{plan}) = \alpha\cdot n_{\text{slide}} + \beta\cdot n_{\text{lift}} + \gamma\cdot n_{\text{dig}},
\end{align}
where $n_{\text{slide}}$, $n_{\text{lift}}$, and $n_{\text{dig}}$ count the number of each action type in the plan, and $[\alpha, \beta, \gamma]^\top\in\mathbb{R}^{3}$ are tunable cost parameters. The dig cost parameter $\gamma$ controls the fundamental trade-off in digout behavior. When $\gamma$ is low, the planner favors traditional digouts that send obstructing pods to other blocks, using multiple robots and adding laden robots to travel lanes. When $\gamma$ is high, the planner favors slide-only plans that shuffle pods within the block using free cells, avoiding travel lane congestion but potentially requiring longer extraction times. Intermediate values produce hybrid plans combining both strategies (Figure~\ref{fig:twoways}). In this work we focus on tuning $\gamma$ while fixing $\alpha$ and $\beta$ based on operational considerations. 

Beyond these parameters, we also have additional parameters that regulates the amount of work that can queue up at each station. Let $\rho$  be a queue size parameter that regulates (relative to a nominal value) the amount of queuing allowed at each station. In this work, we focus on tuning $\gamma$ and $\rho$ while fixing $\alpha$ and $\beta$ based on operational considerations, i.e., $\theta=\{\gamma,\rho\}\in\mathbb{R}^2$.

The A* planner includes additional parameters beyond these three that control other aspects of planning behavior, all captured in the general cost function $F_\theta$. The A* search expands states (configurations of pod locations) and selects the plan with minimum total cost to move the target pod to the boundary. For a given choice of $F_\theta$, the A* planner finds the optimal (minimum cost) individual digout plan. However, our ultimate goal is to maximize throughput: the total number of pod retrievals completed per unit time across all stations. The key challenge is determining which parameter values $\theta_t= \theta(t)$ at each time $t$ will maximize long-term cumulative throughput, even though each individual plan is optimal with respect to its local cost function $F_{\theta_t}$.

\begin{figure*}[thb]
  \centering
  \includegraphics[width=0.6\columnwidth]{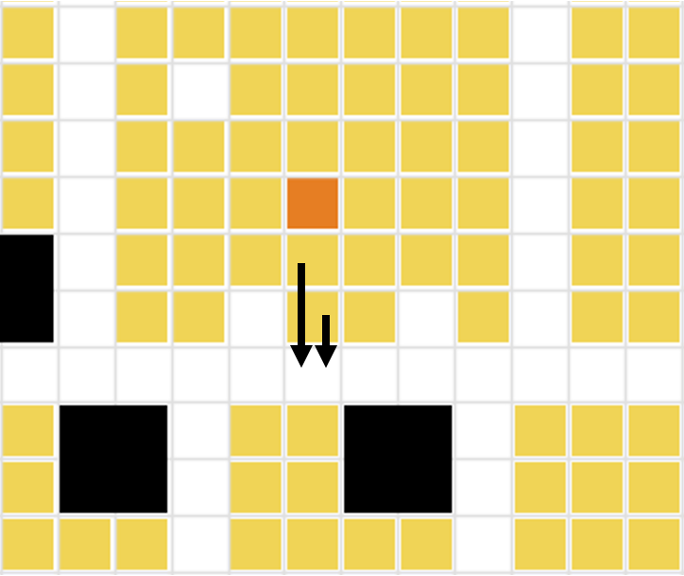}
  \hspace{0.01\columnwidth}
  \includegraphics[width=0.6\columnwidth]{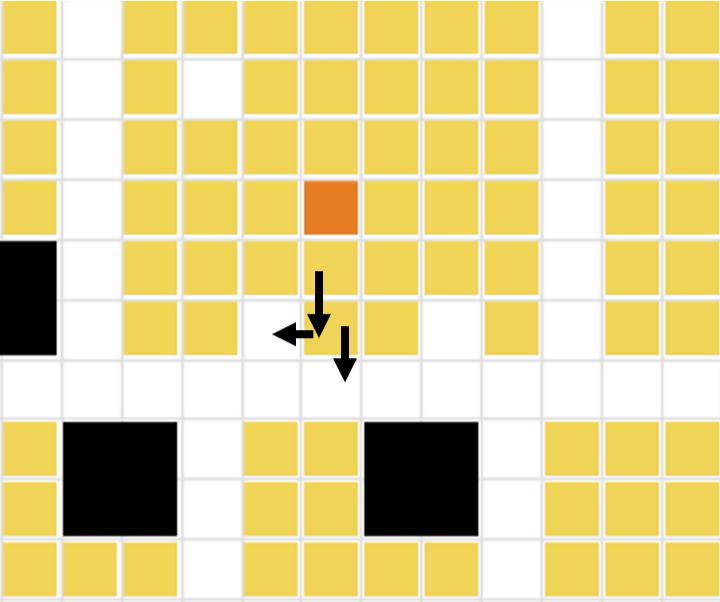}
  \hspace{0.01\columnwidth}
  \includegraphics[width=0.6\columnwidth]{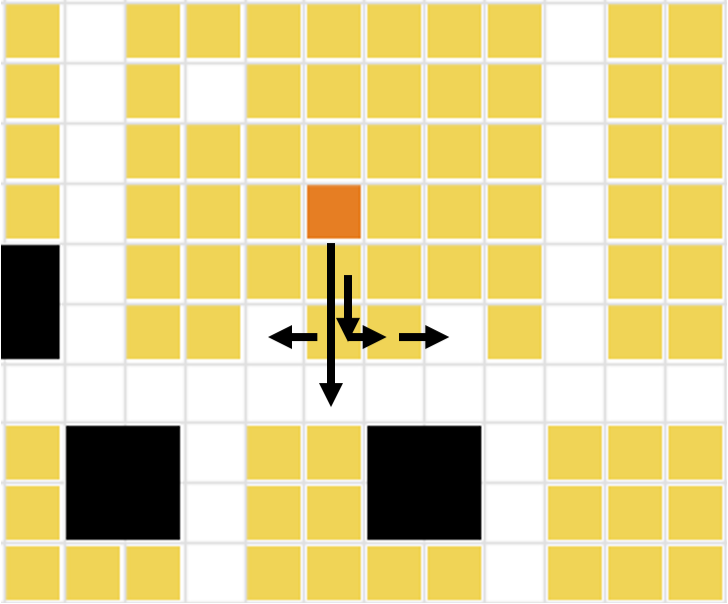}
  \caption{Three behaviors that free up space to extract the target pod (orange) that can result from setting different parameter values in $F_\theta$. The leftmost one involves sending two pods to a different block. The rightmost one sends no pods in the travel lanes and only shuffles block within a block. The middle one involves a combination of the two.}
  \label{fig:twoways}
\end{figure*}

Let $J(\theta,t)$ denote the throughput (pods retrieved per unit time) achieved under parameter $\theta$ at time $t$. The online tuning problem seeks to maximize cumulative throughput over a time horizon $T$:
\begin{align}
\label{eq:opt}
\max_{\theta_t}\;\; \mc J:=
\int_{t}^{t+T} J(\theta_s , s)\,d s
\end{align}
The optimizer does not have access to gradients $\nabla\mc J$ and only observes noisy scalar throughput measurements $J(\theta_t,t)$ during operation. At discrete update times, the tuning law adjusts the parameter:
\begin{align}
\label{eq:projection}
\theta_{t+1} \;=\;  \text{proj}_{\Theta}(\theta_t + \Delta\theta_t),
\end{align}
where $\Delta\theta_t$ is computed by the ESC controller and projection ensures $\theta_{t+1}\in\Theta\subset\mathbb{R}$.

\section{Proposed Method}\label{sec_ot}

We present an adaptive tuning approach based on dither-based ESC for optimizing planner cost parameters in multi-robot digout planning. For ease of notation, we focus on the scalar case where $\gamma$ is the parameter to be tuned with $\alpha$ and $\beta$ held constant; the approach extends naturally to the vector case $\theta\in\R^d$ (refer to Section \ref{MESC}).

\subsection{Continuous-Time Dither-Based ESC}

Consider the continuous-time performance functional $\mc J(\theta)$ that we seek to maximize. Classical dither-based ESC~\cite{ariyur2003real,krstic2000adaptive} employs sinusoidal perturbation and demodulation to estimate the gradient $\nabla\mc J(\theta)$ without explicit differentiation. The control law begins by adding a sinusoidal dither signal $\xi(t) = a\sin(\omega t)$ with amplitude $a>0$ and frequency $\omega>0$ to the parameter estimate $\hat\theta(t)$:
\begin{equation}\label{eq:esc-perturb}
\theta(t) = \hat\theta(t) + \xi(t).
\end{equation}
The measured performance $y(t) = \mc J(\theta(t))$ is then passed through a high-pass filter to remove the DC component:
\begin{equation}\label{eq:esc-hp}
\dot{y}_{\text{hp}}(t) = -\omega_h y_{\text{hp}}(t) + \dot{y}(t),
\end{equation}
where $\omega_h > 0$ is the high-pass cutoff frequency. The filtered signal is multiplied by the dither signal (demodulation) and passed through a low-pass filter:
\begin{equation}\label{eq:esc-demod}
\dot{\xi}_{\text{lp}}(t) = -\omega_l \xi_{\text{lp}}(t) + \omega_l y_{\text{hp}}(t)\sin(\omega t),
\end{equation}
where $\omega_l > 0$ is the low-pass cutoff frequency. Finally, the parameter estimate is updated via:
\begin{equation}\label{eq:esc-int}
\dot{\hat\theta}(t) =\bar{k} \xi_{\text{lp}}(t),
\end{equation}
where $\bar{k}>0$ is the integration gain. Under standard singular perturbation and averaging assumptions~\cite{krstic2000adaptive}, the demodulated signal $\xi_{\text{lp}}(t)$ approximates $(a/2)\nabla\mc J(\hat\theta)$, yielding gradient ascent on $\mc J$.

\subsection{Discrete-Time Implementation}

Since in practice we can update the dig cost $\gamma$ only periodically every $\Delta t$ time units, we consider the discrete time version of the ESC law, where the learning cost is discretized: \begin{align}
\max_{\theta_k}\;\; \mc J:=
\sum_{j=0}^{K-1} J(\theta_k, t_{k+j})\,\Delta t,
\end{align}
where $\theta_k$ is the parameter value during time interval $[t_k, t_{k+1})$, $\Delta t = t_{k+1} - t_k$ is the update period, and $K = T/\Delta t$ is the number of update intervals. 

Let $t_k = k\Delta t$ denote discrete time steps and $y_k = \mc J(\theta_k)$ the measured throughput at step $k$. The high-pass filter is discretized using backward Euler with time constant $\tau_h = 1/\omega_h$:
\begin{equation}\label{eq:esc-hp-disc}
y_k^\text{hp} = h\, (y_{k-1}^\text{hp} + y_k - y_{k-1}),
\end{equation}
where $h = 1/(1 + \omega_h\Delta t)$ is the discrete filter coefficient. This formulation removes the DC component from the throughput signal while preserving variations at the dither frequency. The demodulated signal is computed as:
\begin{equation}\label{eq:esc-demod-disc}
\xi_k = y_k^\text{hp}\sin(\omega t_k).
\end{equation}
The low-pass filter is discretized using exponential smoothing with time constant $\tau_l = 1/\omega_l$:
\begin{equation}\label{eq:esc-lp-disc}
\xi_{k}^\text{lp} = \alpha_l \xi_k + (1-\alpha_l)\xi_{k-1}^\text{lp},
\end{equation}
where $\alpha_l = \omega_l\Delta t/(1 + \omega_l\Delta t)$. The discrete integration step becomes $\hat\theta_{k+1} = \hat\theta_k + \bar{k}\,\xi_{k}^\text{lp}\,\Delta t$,
and the applied parameter is $\theta_k = \hat\theta_k + a\sin(\omega t_k)$.
The parameter $\theta_k$ is clamped to the feasible set $\Theta$ via projection. To account for the system delay $\tau$ between parameter changes and throughput response, the demodulation signal can be phase-shifted:
\begin{equation}
    \xi_k = y_k^\text{hp} \sin(\omega(t_k - \tau)).
    \label{eq:phase_comp}
\end{equation}
In the next section, we discuss how this phase change can be measured.

\subsection{Multi-Parameter Extension}\label{MESC}
The proposed framework extends naturally to vector-valued parameters $\boldsymbol{\theta} \in \mathbb{R}^d$ by using orthogonal dither frequencies. For $d$ parameters, we employ:
\begin{equation}
    \theta_i(t) = \hat{\theta}_i(t) + a \sin(\omega_i t), \quad i = 1, \ldots, d,
\end{equation}
where the frequencies $\{\omega_i\}$ are mutually distinct and sufficiently separated (e.g., $\omega_i = \omega_1 \cdot r^{i-1}$ for irrational $r$). The demodulation for each parameter uses its corresponding frequency: $\xi_{\text{lp},i} = \text{LPF}\left[ y_{\text{hp}} \sin(\omega_i t) \right]$
ensuring that cross-coupling between parameter updates averages to zero over time. 

\subsection{Convergence of Dither-Based ESC}\label{sec:esc-conv}
We analyze the convergence of the dither-based extremum seeking controller (ESC) used to tune the parameters $\theta\in\R^d$. For the analysis, we make the following two assumptions.
\paragraph*{Assumption A1}
\begin{inparaenum}
\item $\Theta$ is convex and compact; A* with $F_\theta$ is complete w.r.t.\ free space on $G$.
\item $\mc J$ is bounded, smooth and locally Lipschitz in $\theta\in\Theta$.
\item $\mc J\in C^3$ in a neighborhood $\mc N$ of a local maximizer $\theta^\star$; $\nabla \mc J(\theta^\star)=0$ and $-H^\star \prec 0$, where $H^\star:=\nabla^2 \mc J(\theta^\star)$.
\item Episode lengths $T$ are uniformly bounded, and safety monitors prevent deadlock by backoff (finite recovery).
\item Measurements of $J(\theta,t)$ may be noisy but have bounded variance.
\end{inparaenum}

\paragraph*{Assumption A2}
\begin{inparaenum}
\item Frequency $\omega>0$ satisfies $\omega\gg \omega_l$; the low-pass removes sum/difference harmonics.
\item The dither amplitude $a>0$ is small such that the higher-order terms $O(a^2)$ and filtering remainders $O(1/\omega)$ are negligible.
\end{inparaenum}

\begin{lem}
[Local practical convergence of ESC]
\label{thm:esc}
Under Assumptions A1, A2, and the dither ESC control law (\ref{eq:esc-perturb})-(\ref{eq:esc-int}) for sufficiently small $a$ and sufficiently large $\omega$, there exist $c_1,c_2,\varepsilon>0$ such that if $\|\hat\theta(0)-\theta^\star\|$ is small, then
\[
\|\hat\theta(t)-\theta^\star\| \;\le\; c_1 e^{-c_2 t}\,\|\hat\theta(0)-\theta^\star\| \;+\; O(\varepsilon),
\]
with ultimate bound $O\big(a+ \omega^{-1}\big)$. In particular, as $a\!\to\!0$ and $\omega\!\to\!\infty$, the trajectories converge to the local maximizer $\theta^\star$.
\end{lem}
\begin{proof}
See \cite{Benosman2016}.
\end{proof}

\begin{rmk}
In this empirical paper, the primary goal is to demonstrate how ESC
can be efficiently incorporated into a large-scale discrete-event
system. Nevertheless, we recall the main convergence result of a
continuous ESC loop for a single variable to state the sufficient
conditions of the algorithm. Extensions to multi-variable continuous
ESC are well established \cite{ariyur2003real,Benosman2016}.
Convergence guarantees under discretization still hold; see, e.g.,
\cite{Poveda2017,Sanfelice2010}. The projection step in
\eqref{eq:projection} onto the convex compact set $\Theta$ does not
affect convergence, as constrained ESC on compact sets is naturally
handled within the hybrid ESC framework of \cite{Poveda2017}.
\end{rmk}

\section{Experimental Validation}\label{sec_sims}

\subsection{Experimental Setup}
We evaluate the proposed ESC controller using a discrete-event simulator that models an FC that implements robot kinematics, pod handling dynamics, and station request patterns that broadly model FC behavior. The simulator tracks individual robot states, pod locations, and station queues to compute throughput and other operational metrics.
We use two floors, termed MAP1 and MAP2, that represent a smaller (hundreds of robots) and a larger floor (more than a thousand robots). Each simulation runs for 8 hours of simulated time, with the first 10 minutes excluded from analysis to allow the system to reach steady state. We initialize the ESC with $\hat\theta_0 = \theta_0$, $y_{\text{hp},0} = 0$, and $\xi_{\text{lp},0} = 0$. 

Recall from Section~\ref{sec:Intro} the four key challenges in parameter optimization: FC-specific optimal configurations, expensive throughput evaluation, delayed effects with difficult credit assignment, and lack of closed-form models. To address these challenges, we design four experiments corresponding to research questions RQ1--RQ4. RQ1 validates that parameter perturbations produce observable throughput changes, enabling gradient estimation. RQ2 characterizes the phase delay to inform update timing and address credit assignment. RQ3 evaluates whether adaptive policies can eliminate manual parameter sweeps by automatically converging to good configurations. RQ4 examines adaptation to dynamic operational changes.

ESC parameters for experiments reported under RQ3 and RQ4 were selected to satisfy timescale separation requirements after preliminary analysis in RQ1 and RQ2. The dither frequency $\omega = 0.002$ rad/s (period $\approx 52$ min) was chosen to be significantly slower than the 4-minute phase delay identified in RQ2. For experiments involving simultaneous adaptation of both parameters (transfer analysis and RQ4), we use frequency-separated dither signals with $\omega_\rho = 0.0028$ rad/s and $\omega_\gamma = 0.0007$ rad/s to avoid cross-coupling, as described in Section~\ref{MESC}. The high-pass filter time constant 
$\tau_h = 42$ minutes removes slow throughput drift caused by gradual floor  state changes unrelated to parameter perturbations, while the low-pass filter  time constant $\tau_l = 24$ minutes (approximately half the dither period)  smooths noisy gradient estimates by averaging over sufficient dither cycles. 
These time constants satisfy the required separation $\tau_h > \tau_l > 1/\omega$. The amplitude was chosen to be $10\%$ of the range of the parameter which was validated via RQ1 to produce observable throughput oscillations. 
The integration gain $\bar{k} = 0.0003$ was tuned to balance convergence speed against stability, yielding convergence within the 8-hour simulation horizon.

\subsection{RQ1: Does Dither in Control Input Manifest in Throughput?}

A fundamental requirement for ESC is that perturbations in the control parameter\footnote{In this experiment, we focus only on dig cost; the results are similar for $\rho$.} must produce observable changes in the performance metric (throughput). We analyze this using Power Spectral Density (PSD) analysis of both signals during ESC operation.

Figure~\ref{fig:psd} shows the PSD plots for dig cost and throughput signals. The throughput signal shows a  local peak $\geq 15$dB above the noise floor at the same frequency as the dig cost signal, demonstrating that the dither in dig cost propagates through the system and manifests as observable oscillations in throughput. This validates that the system is responsive to parameter changes and that ESC can extract gradient information from throughput measurements.

The throughput PSD also reveals a strong DC component (low-frequency content) and some high-frequency noise. The DC component represents the baseline throughput level and is removed by the high-pass filter in the ESC controller. The high-frequency noise is attenuated by the low-pass filter, ensuring that gradient estimates are based primarily on the dither-induced signal component.

\begin{figure}[t]
  \centering
  \begin{subfigure}[b]{0.46\textwidth}
    \centering
    \includegraphics[width=\textwidth]{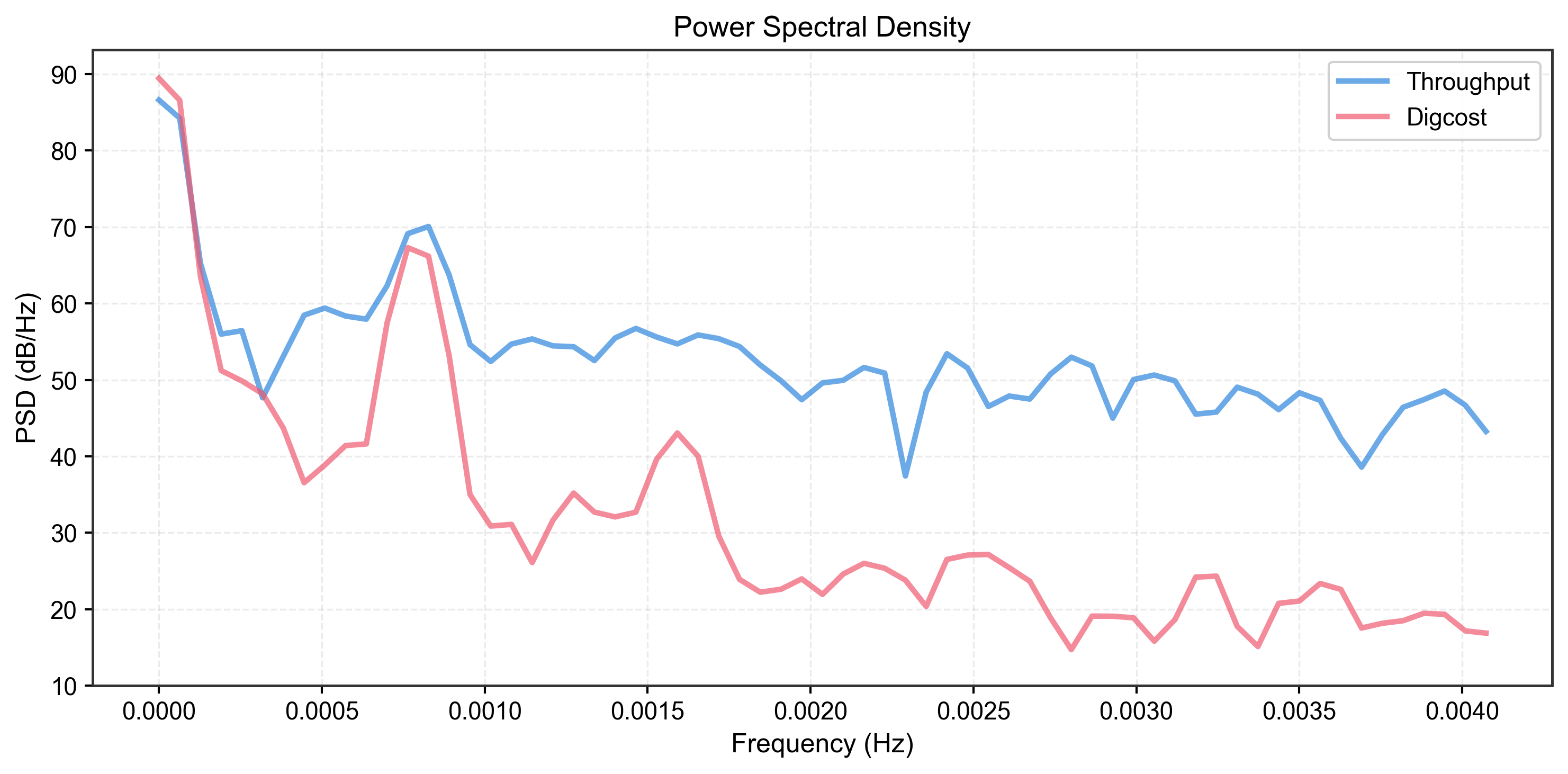}
    \caption{Power spectral density analysis showing that dither in dig cost produces corresponding oscillations in throughput at the same frequency, validating system responsiveness for ESC.}
    \label{fig:psd}
  \end{subfigure}
  \hfill
  \begin{subfigure}[b]{0.46\textwidth}
    \centering
    \includegraphics[width=\textwidth]{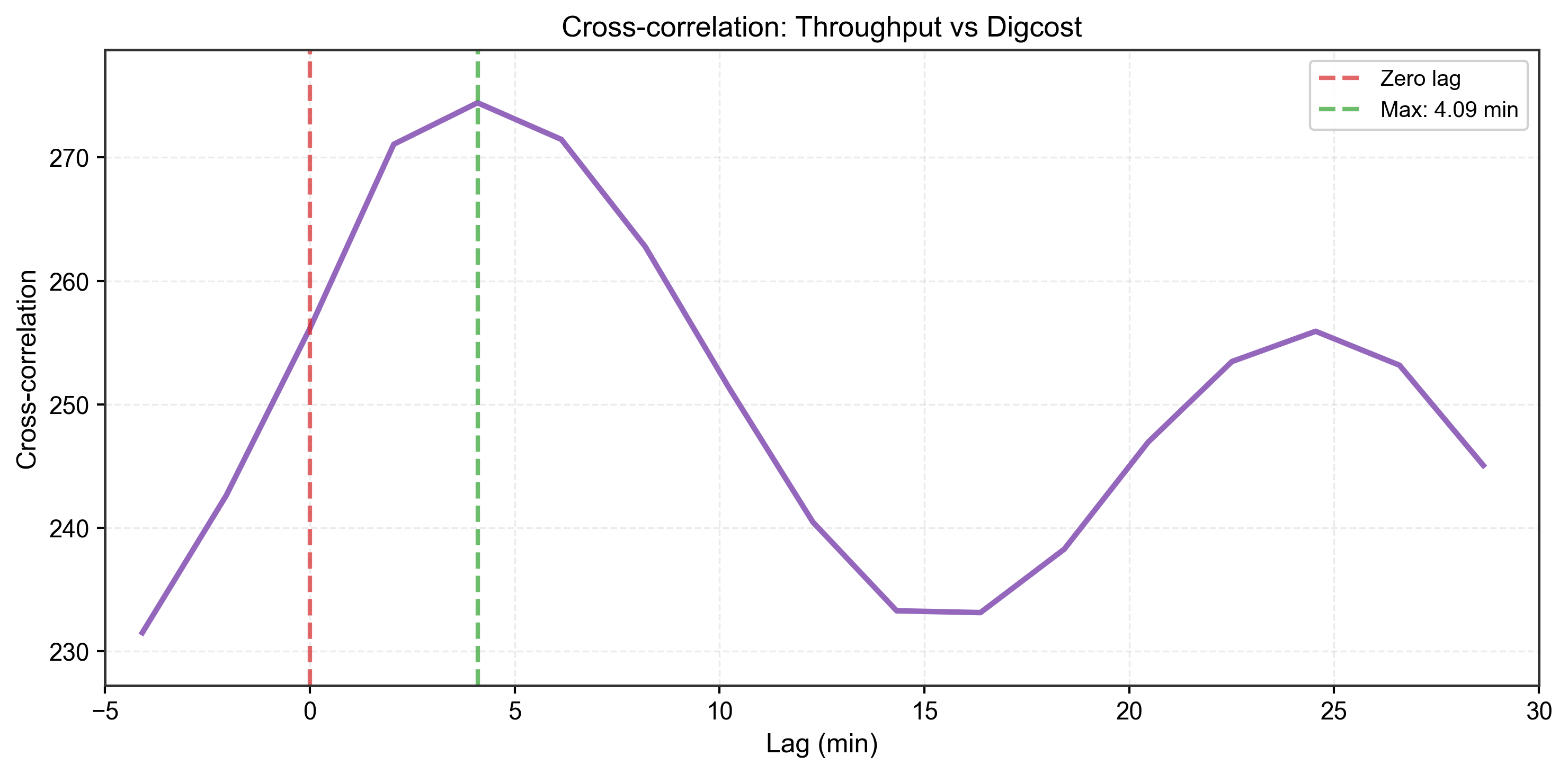}
    \caption{Cross-correlation analysis showing peak correlation at approximately 4 minutes, indicating the phase delay between dig cost changes and throughput response.}
    \label{fig:phase}
  \end{subfigure}
  \caption{System characterization: (a) frequency response validation and (b) phase delay analysis.}
  \label{fig:system_char}
\end{figure}

\subsection{RQ2: What is the Phase Delay Between Cost and Throughput?}

The effectiveness of ESC depends on proper timing between parameter updates and performance measurements. If throughput is measured before the system responds to a dig cost change, gradient estimates will be corrupted by transient effects or outdated parameter values.

We characterize the phase delay using cross-correlation analysis between dig cost and throughput signals. Figure~\ref{fig:phase} shows the correlation strength as a function of time shift. The peak correlation occurs at approximately $4$ minutes, indicating that throughput changes lag dig cost changes by this amount. This delay represents the time required for a change in dig cost to propagate through the planning decisions, robot allocations, and pod movements before affecting measured throughput. 

Based on this analysis, we set the ESC update period to $\Delta t = 5$ minutes. This ensures that when we measure throughput to compute gradient estimates, the system has had sufficient time to settle into steady-state behavior under the current dig cost. Updating more frequently would risk measuring transient effects or attributing throughput changes to the wrong parameter values.

\subsection{RQ3: Performance Under Nominal Conditions}

We evaluate whether the adaptive ESC policy can match or exceed the performance of fixed policies under nominal operating conditions. In our first test, we vary the queue size parameter $\rho \in \{1, 2, 3\}$ used to initialize the ESC policy. The initial dig cost $\gamma$ is set fixed to 0.4 in these experiments. Both the parameters are adapted using the ESC controller described in the previous section. We evaluate across five different robot counts (200, 300, 400, 500, 600) to assess scalability of the method. Each experiment was run for 8 hours of simulated time with 5 random seeds to ensure statistical reliability.

\begin{figure}[t]
  \centering
  \includegraphics[width=\columnwidth]{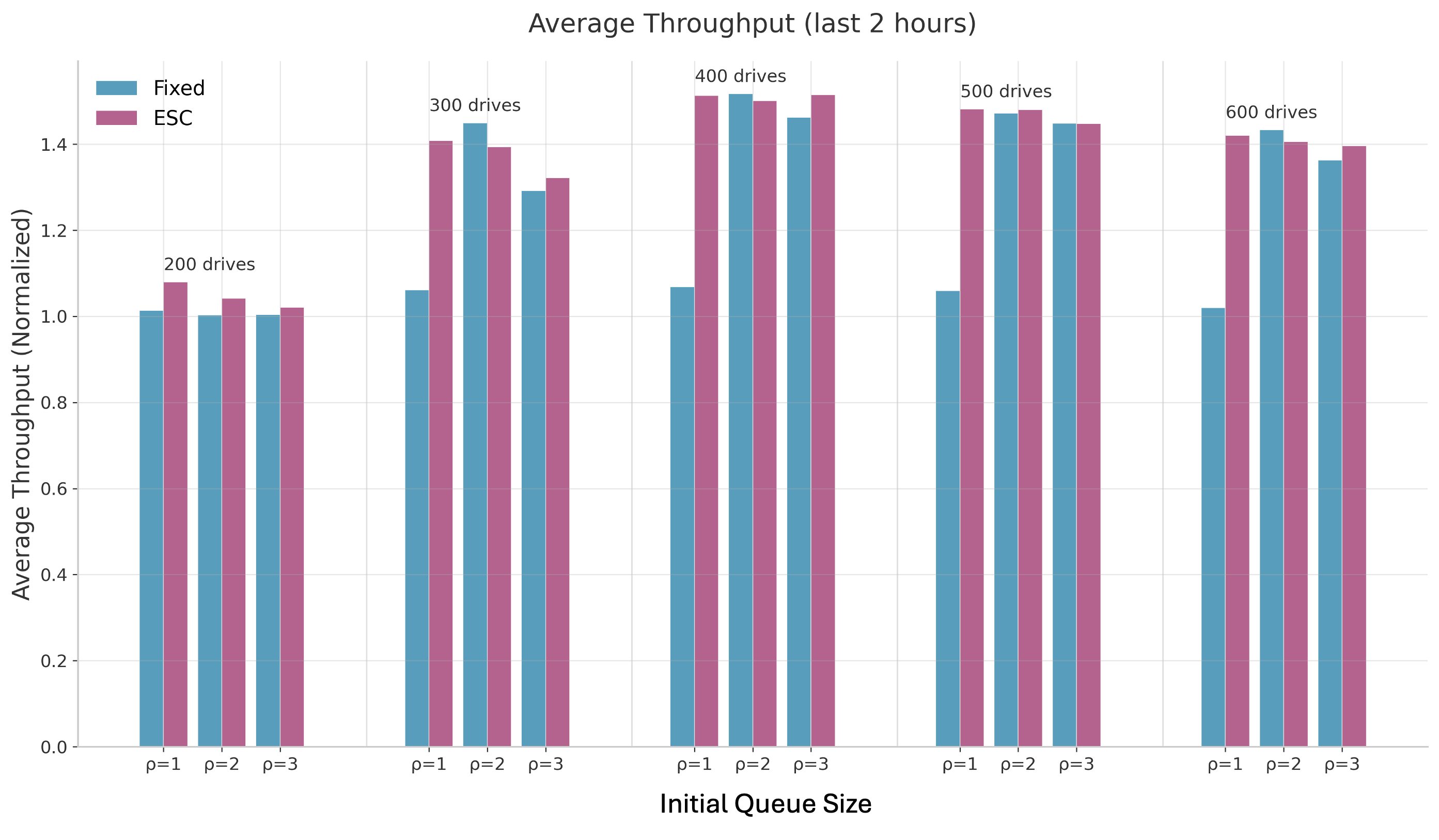}
  \caption{Throughput comparison during the last 2 hours of simulation. ESC achieves an average improvement of 11.16\% over fixed policies in steady-state, demonstrating that adaptive benefits compound as the controller converges to optimal parameter configurations.}
  \label{fig:throughput_last2h}
\end{figure}

\begin{figure}[t]
  \centering
  \includegraphics[width=\columnwidth]{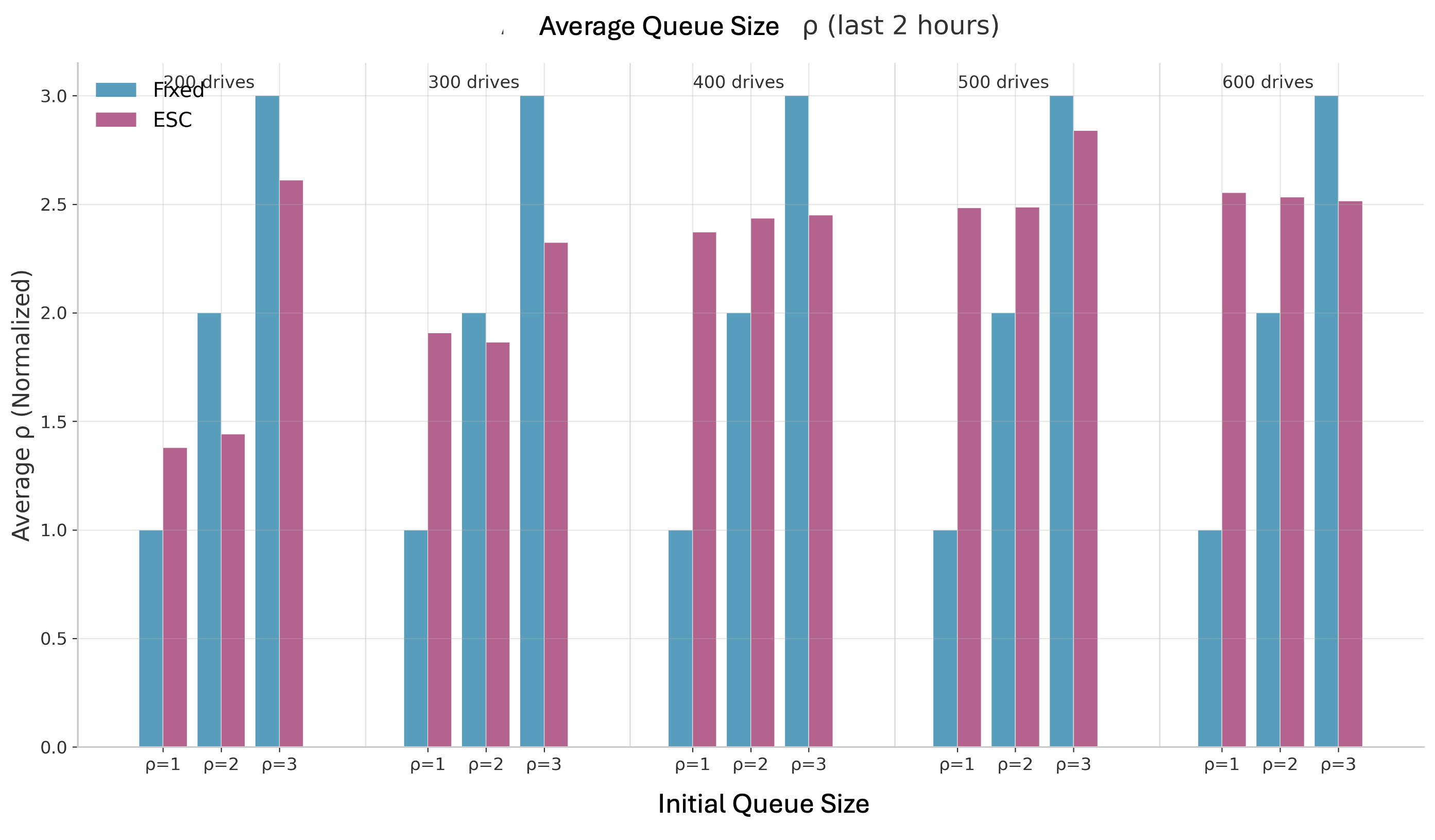}
  \caption{Convergence of queue size parameter $\rho$ across all experimental conditions. ESC consistently drives $\rho$ toward optimal values regardless of initialization, with final converged values showing clear patterns: $\rho \approx 1.6-2.0$ for most conditions, demonstrating automatic parameter discovery without manual tuning.}
  \label{fig:rho_convergence}
\end{figure}

Figure~\ref{fig:throughput_last2h} presents the throughput comparison between fixed and ESC dither policies across all experimental conditions. The figure shows data from the final 2 hours of each 8-hour simulation, representing the system's behavior after ESC adaptation.  The results demonstrate that ESC consistently outperforms fixed policies, with the magnitude of improvement varying by operational regime, the initial $\rho$ setting, and the system scale.

Figure~\ref{fig:rho_convergence} demonstrates ESC's ability to automatically discover optimal parameter configurations across all experimental conditions. The figure shows the average queue size $\rho$ during the final 2 hours of simulation, representing the converged parameter values after 8 hours of ESC adaptation. Regardless of the initial $\rho$ value (1, 2, or 3), ESC consistently drives the parameter toward a narrow optimal range of approximately $\rho = 1.6-2.0$ across different robot counts. The tight clustering of final $\rho$ values around the optimal range, independent of initialization, provides strong evidence that ESC discovers true system optima rather than merely improving upon poor initial guesses.

At low queue sizes ($\rho = 1$), ESC achieves substantial improvements of 24-42\% over fixed policies, indicating that adaptive parameter tuning provides the greatest benefit when the system operates far from optimal initial conditions. At medium queue sizes and high queue sizes ($\rho = 2$), the performance difference is more modest, suggesting that the initial parameter choice is closer to optimal. The slight performance degradation observed in these regimes is expected when ESC is initialized near optimal parameter values, as the dither signal necessary for gradient estimation introduces small perturbations around the optimum that can temporarily reduce performance.

To illustrate the adaptive behavior of ESC, Figure~\ref{fig:esc_adaptation} shows the time-series evolution of both queue size $\rho$ alongside throughput for a representative experiment (300 robots, initial $\rho = 1$). The figure illustrates an example of how ESC automatically adjusts the queue size from its initial value of 1 to an optimal range around 2, corresponding to the system's learned preference for moderate queuing levels that balance station utilization with congestion avoidance. 

\begin{figure}[t]
  \centering
  \includegraphics[width=\columnwidth]{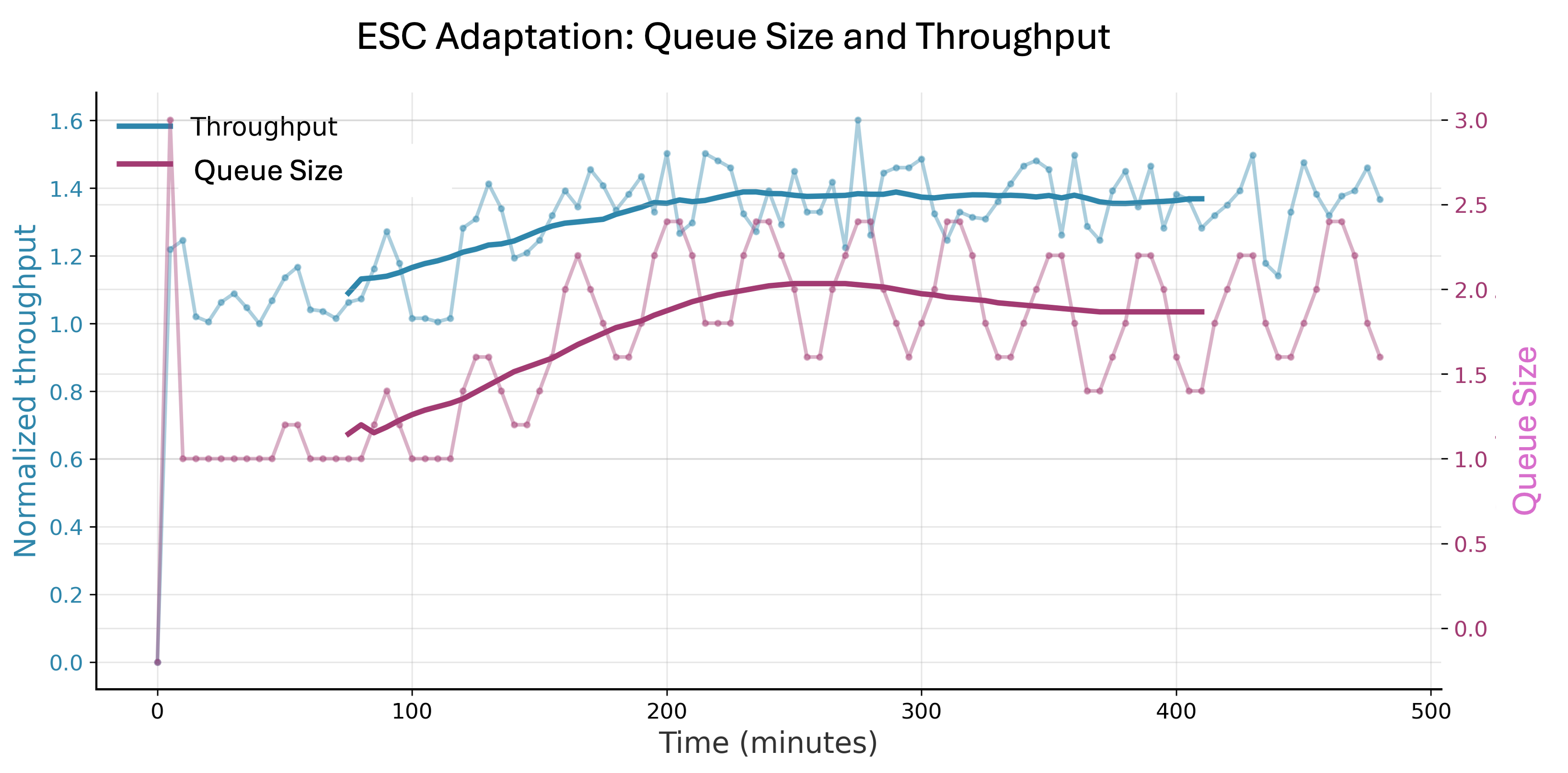}
  \caption{Time-series evolution of ESC parameter adaptation for a representative experiment (300 robots, initial $\rho = 1$): Queue size $\rho$ adaptation showing convergence from initial value of 1 to optimal range around 2. Note that we plot the normalized $\rho$ value; the true, unnormalized $\rho$ value used in the simulations is an integer which is why the sinusoidal dither looks rounded off.}
  \label{fig:esc_adaptation}
\end{figure}

In the experiments with MAP1, the choice of dig cost does not affect the throughput significantly. However, MAP2 is a much larger warehouse with significantly higher number of robots (1600) than MAP1. MAP2 shows much more sensitivity to the choice of dig cost $\gamma$. Table~\ref{tab:normalized_results} presents normalized throughput comparison between the fixed and the ESC policy on MAP2 as a function of the initial dig cost. When initialized far from optimal parameters (normalized dig cost = 0.4), ESC achieves substantial improvements of 3-4\%, demonstrating effective parameter discovery. When initialized closer to optimal values (normalized dig cost = 0.5-0.9), ESC shows modest but consistent improvements of 1-2\%, with the controller maintaining stability when gradients are near zero.
\begin{table}[t]
\centering
\small
\caption{Normalized throughput and dig cost comparison between fixed and ESC policies on MAP2 with 1600 robots.}
\label{tab:normalized_results}
\begin{tabular}{cccc}
\toprule
\textbf{Initial Dig Cost} & \textbf{Fixed} & \textbf{ESC} & \textbf{Improvement} \\
\textbf{(Normalized)} & \textbf{Throughput} & \textbf{Throughput} & \textbf{(\%)} \\
\midrule
0.4 & 1.00 & 1.035 & 3.5 \\
0.5 & 1.04 & 1.075 & 3.4 \\
0.6 & 1.02 & 1.038 & 1.8 \\
0.7 & 1.03 & 1.048 & 1.8 \\
0.8 & 1.03 & 1.048 & 1.8 \\
0.9 & 1.04 & 1.044 & 0.6 \\
\bottomrule
\end{tabular}
\end{table}


\paragraph*{Transfer Regret Analysis}
A practical concern with fixed policies is that parameters optimized for one FC configuration may not transfer well to different configurations. As stated earlier, manually running sweeps of the parameters to find the best configuration for each new map is not practical. Instead one may pick representative configurations to find the best parameters on, and then use them in other similar configurations. 

To quantify this, we evaluate the transfer regret of fixed policies: the throughput loss incurred when deploying parameters tuned on one configuration (the ``train'' setting) to a different configuration (the ``test'' setting), compared to an ESC policy that adapts online. We consider three transfer scenarios: (i) train and test use the same map (MAP1) but different number of robots, (ii) train and test use different maps but same number of robots, and (iii) train and test use different map and different number of robots. The fixed policy uses parameters optimized for 300 robots on MAP1; the ESC policy starts from the same initial parameters but adapts online during each test run. The test configurations use 200 and 400 robots and version of MAP1 that has larger individual blocks (i.e., the test map has blocks that are $(N+1)\times (N+1)$ instead of $N\times N$).

Table~\ref{tab:transfer} summarizes the results. Each entry reports the mean throughput improvement of ESC over the transferred fixed policy, computed per-seed across 5 random seeds per configuration, along with the 95\% confidence interval. Across all 25 unique test seeds, ESC improves throughput by an average of $5.0\%$ (95\% CI: $[2.6\%, 7.4\%]$). The worst-case ESC loss across all seeds is $-3.0\%$, while the best-case gain is $+17.3\%$. The ESC loss can be attributed to the dither that intentionally perturbs the system from the optimal configuration. That marginal loss is amply compensated by the significant gains over using fixed configurations.

\begin{table}[t]
\centering
\small
\caption{Transfer regret analysis: mean ESC improvement over fixed policy with 95\% confidence intervals. Fixed policy uses parameters tuned for 300 robots on MAP1.}
\label{tab:transfer}
\begin{tabular}{lcc}
\toprule
\textbf{Transfer Scenario} & \textbf{Mean $\Delta$\%} & \textbf{95\% CI} \\
\midrule
Same map, diff.\ num.\ robots & $+4.2\%$ & $[+2.4\%, +5.9\%]$ \\
Diff.\ map, same num.\ robots & $+5.6\%$ & $[+1.4\%, +9.7\%]$ \\
Diff.\ map, diff.\ num.\ robots & $+4.0\%$ & $[-1.4\%, +9.3\%]$ \\
\bottomrule
\end{tabular}
\end{table}

These results highlight that fixed policies incur non-trivial transfer regret when deployed outside their training configuration, particularly when the map changes. ESC's ability to adapt online eliminates this transfer cost, providing robust performance across diverse operating conditions without requiring configuration-specific parameter sweeps.

\subsection{RQ4: Performance Under Dynamic Conditions}
A key motivation for adaptive control is the ability to respond to changing operational conditions without manual intervention. We simulate a scenario on MAP2 (1600 robots) where 50\% of pick stations log off at $t = 150$ minutes, representing a common operational event such as a shift change. Figure~\ref{fig:dynamic_comparison} shows a representative run where the fixed policy maintains constant $\rho$ while throughput drops sharply after the station reduction, whereas the adaptive ESC policy adjusts $\rho$ in response to the changed conditions.

To quantify the effect, we run 5 paired experiments with randomized logout times uniformly sampled between 150 and 250 minutes, controlling for the phase relationship between the dither signal and the disturbance. Table~\ref{tab:dynamic} summarizes the averaged results. Pre-logout, ESC achieves a modest $+3.5\%$ improvement from online tuning of $\rho$. Post-logout, the advantage increases to $+14.3\%$ as ESC adapts $\rho$ from approximately $3.3$ to $4.3$ in response to the reduced station capacity, while the fixed policy remains at $\rho = 3$. Over the full 8-hour simulation, ESC improves throughput by $+8.4\%$. Note that the dig cost $\gamma$ does not change significantly across these runs, as the adaptation is driven primarily by $\rho$.

\begin{table}[t]
\centering
\small
\caption{Dynamic conditions: mean ESC improvement over fixed policy across 5 experiments with randomized station logout times (150--250 min). MAP2, 1600 robots, 50\% station logout.}
\label{tab:dynamic}
\begin{tabular}{cccccc}
\toprule
& \textbf{$\Delta$\% pre} & \textbf{$\Delta$\% post} & \textbf{$\Delta$\% full} & \textbf{$\rho$ pre} & \textbf{$\rho$ post} \\
\midrule
\textbf{Mean} & $+3.5$ & $+14.3$ & $+8.4$ & $3.3$ & $4.3$ \\
\bottomrule
\end{tabular}
\end{table}


\begin{figure}[h]
    \centering
    \includegraphics[width=\columnwidth]{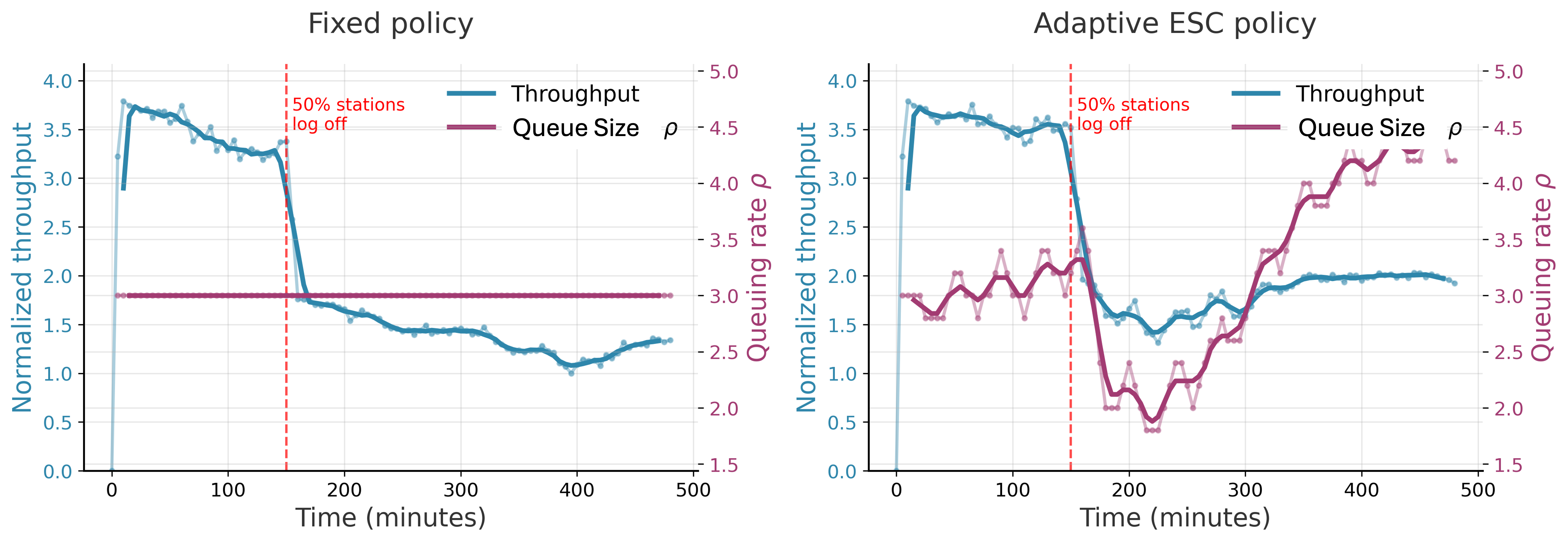}
    \caption{Representative response to 50\% station reduction on MAP2 with 1600 robots. The fixed policy (left) maintains constant $\rho$ while throughput drops. The adaptive ESC policy (right) increases $\rho$ after the disturbance, recovering higher throughput.}
    \label{fig:dynamic_comparison}
\end{figure}

\section{Conclusions}\label{sec_conc}

We presented an adaptive parameter tuning framework for a multi-robot planner in robotic fulfillment centers. The framework addresses two key challenges in dense storage operations: eliminating time-consuming manual parameter sweeps across different FC configurations, and enabling real-time adaptation to dynamic operational conditions. Our approach uses ESC, a model-free adaptive control technique that continuously adjusts planner parameters in response to measured throughput. By perturbing the parameter with a sinusoidal dither signal and correlating perturbations with throughput changes, ESC estimates gradients without requiring an analytical model of the complex relationship between planning parameters and system-level performance. This is relevant given the delayed effects and credit assignment challenges inherent in FC operations, where local planning decisions ripple through the floor over time. Through simulation studies using realistic FC configurations, we demonstrated that: (1) parameter perturbations manifest as observable throughput changes, validating the feasibility of gradient-based adaptation; (2) the $4$-minute phase delay between cost changes and throughput response informs appropriate update timing; (3) the adaptive policy improves upon fixed policies.

This work bridges search-based planning with adaptive control, providing a self-tuning paradigm for dense FC operations. The ESC framework generalizes beyond the two planning and allocation parameters considered in this work and could be applied to other parameters or planning systems where the relationship between control parameters and performance metrics is complex and time-varying.  

\bibliographystyle{IEEEtran}
\bibliography{refs}


\end{document}